\documentclass{article}

\usepackage{microtype}
\usepackage{graphicx}
\usepackage{subfigure}
\usepackage{booktabs} 
\usepackage{adjustbox}
\usepackage{hyperref}
\usepackage{enumerate}

\usepackage[accepted]{icml2023}

\usepackage{amsmath}
\usepackage{amssymb}
\usepackage{mathtools}
\usepackage{amsthm}

\usepackage[capitalize,noabbrev]{cleveref}

\newcommand{\bq}{\boldsymbol{q}}
\newcommand{\bg}{\boldsymbol{g}}
\newcommand{\bs}{\boldsymbol{s}}
\newcommand{\bx}{\boldsymbol{x}}
\newcommand{\bu}{\boldsymbol{u}}
\newcommand{\by}{\boldsymbol{y}}

\newcommand{\bz}{\boldsymbol{z}}

\newcommand{\btheta}{\boldsymbol{\theta}}

\newcommand{\argmin}{\mathop{\mathrm{argmin}}}

\newcommand{\interior}{\mathop{\mathrm{int}}}
\newcommand{\dom}{\mathop{\mathrm{dom}}}

\newcommand{\field}[1]{\mathbb{#1}}

\newcommand{\R}{\field{R}}

\DeclareMathOperator{\Regret}{Regret}

\DeclareMathOperator{\Prox}{Prox}

\theoremstyle{plain}
\newtheorem{theorem}{Theorem}[section]

\newtheorem{lemma}[theorem]{Lemma}
\newtheorem{cor}[theorem]{Corollary}
\theoremstyle{definition}

\theoremstyle{remark}

\icmltitlerunning{Generalized Implicit Follow-The-Regularized-Leader}

\begin{document}

\twocolumn[
\icmltitle{Generalized Implicit Follow-The-Regularized-Leader}

\icmlsetsymbol{equal}{*}

\begin{icmlauthorlist}
\icmlauthor{Keyi Chen}{yyy}
\icmlauthor{Francesco Orabona}{yyy}
\end{icmlauthorlist}

\icmlaffiliation{yyy}{Boston University, Boston, MA, USA}

\icmlcorrespondingauthor{Keyi Chen}{keyichen@bu.edu}
\icmlcorrespondingauthor{Francesco Orabona}{francesco@orabona.com}

\vskip 0.3in
]

\printAffiliationsAndNotice{}  

\begin{abstract}
We propose a new class of online learning algorithms, generalized implicit Follow-The-Regularized-Leader (FTRL), that expands the scope of FTRL framework. Generalized implicit FTRL can recover known algorithms, as FTRL with linearized losses and implicit FTRL, and it allows the design of new update rules, as extensions of aProx and Mirror-Prox to FTRL.
Our theory is constructive in the sense that it provides a simple unifying framework to design updates that directly improve the worst-case upper bound on the regret. The key idea is substituting the linearization of the losses with a Fenchel-Young inequality. We show the flexibility of the framework by proving that some known algorithms, like the Mirror-Prox updates, are instantiations of the generalized implicit FTRL. Finally, the new framework allows us to recover the temporal variation bound of implicit OMD, with the same computational complexity.
\end{abstract}

\section{Introduction}
\label{sec:intro}

Online learning is a setting where the learner receives an arbitrary sequence of loss functions, selects points before knowing the loss functions, and is evaluated on the values of the loss functions on the points it selects~\citep{Cesa-BianchiL06,Orabona19,Cesa-BianchiO21}. 
More in detail, at round $t$ the learner outputs a point $\bx_t$ in a feasible set $V\subseteq \R^d$. Then, it receives a loss function $\ell_t: V \to \R$ and it pays the value $\ell_t(\bx_t)$. Given the arbitrary nature of the losses, the learner cannot guarantee to have a small cumulative loss, $\sum_{t=1}^T \ell_t(\bx_t)$. On the other hand, it is possible to minimize the \emph{regret}, that is the difference between the cumulative loss of the algorithm and the one of any arbitrary comparator $\bu \in V$:
\[
\Regret_T(\bu)
\triangleq \sum_{t=1}^T \ell_t(\bx_t) -\sum_{t=1}^T \ell_t(\bu)~.
\]
In particular, a successful online learning algorithm must guarantee a regret that grows sublinearly in time for any $\bu \in V$. In this way, its average performance approaches the one of the best comparator in hindsight.

There are two families of online learning algorithms: Online Mirror Descent (OMD)~\citep{NemirovskijY83,Warmuth97} and Follow-the-Regularized-Leader (FTRL)~\citep{Shalev-Shwartz07,AbernethyHR08,HazanK08}. They stem from two similar but complementary approaches: the update of OMD aims at minimizing a linearization of the current loss without going too far from its previous prediction $\bx_t$, while FTRL minimizes the sum of all the losses (or their linear approximation) plus a regularization term. On the contrary to the first approaches in online learning that focused on specific algorithms (e.g., the Winnow algorithm~\citep{Littlestone88}), the theory of these two frameworks is particularly interesting because it allows \emph{both the design and the analysis of generic online learning algorithms}.

While FTRL and OMD provide similar bounds in most situations, they are not completely equivalent. For example, FTRL has an advantage over OMD in unbounded domains, where it allows to use time-varying regularizers. In fact, OMD allows the use of time-varying stepsizes only in domains where its associated Bregman divergence is bounded.

On the other hand, in the cases where we can use time-varying stepsizes, OMD can achieve a superior adaption to the gradients (see, e.g., Theorem~2 in \citet{StreeterM10} versus Theorem~2 in \citet{OrabonaP15}). In this view, these two frameworks are complementary.\footnote{See also the blog post on this topic by Tim van Erven at \url{https://www.timvanerven.nl/blog/ftrl-vs-omd/}.}
Moreover, there exists another orthogonal axis on the use of the actual loss functions or a linear surrogate for both frameworks.
We summarize all the variants of OMD and FTRL in Table~\ref{table:summary}.

\begin{table*}[t]
\centering
\caption{Summary of implicit and linearized updates for FTRL and OMD. (The Bregman divergence $B_\psi(\bx; \by)$ is defined as $\psi(\bx)-\psi(\by)-\langle \nabla \psi(\by), \bx-\by\rangle$. The $^\star$ denotes the Fenchel conjugate.)}
\label{table:summary}
\adjustbox{max width=\textwidth}{
\begin{tabular}{| l | c|}
\hline
Algorithm & Update \\
\hline
OMD~\citep{Warmuth97} & $\bx_{t+1}=\argmin_{\bx \in V} \ B_\psi(\bx;\bx_t) + \eta_t(\ell_t(\bx_t) + \langle \bg_t, \bx- \bx_t\rangle)$\\
Implicit OMD~\citep{Warmuth97} & $\bx_{t+1}=\argmin_{\bx \in V} \ B_\psi(\bx;\bx_t) + \eta_t \ell_t(\bx)$ \\ 
\hline
FTRL (linearized)~\citep{AbernethyHR08} & $\bx_{t+1}=\argmin_{\bx \in V} \ \psi_{t+1}(\bx) + \sum_{i=1}^{t} (\ell_i(\bx_i) + \langle \bg_i, \bx-  \bx_i\rangle)$ \\
FTRL (full losses)~\citep{McMahan17} & $\bx_{t+1}=\argmin_{\bx \in V} \ \psi_{t+1}(\bx) + \sum_{i=1}^{t} \ell_i(\bx)$ \\ 
Implicit FTRL~\citep{McMahan10} & $\bx_{t+1}=\argmin_{\bx \in V} \ \psi_{t+1}(\bx) + \ell_t(\bx) + \sum_{i=1}^{t-1} (\ell_i(\bx_i) + \langle \bg_i, \bx-  \bx_i\rangle)$\\
\hline
Generalized Implicit FTRL [This work] & $\bx_{t+1}=\argmin_{\bx \in V} \ \psi_{t+1}(\bx) + \sum_{i=1}^t \langle \bz_i, \bx\rangle$ \\
& $\bz_i$ such that $\psi^\star_{i+1,V}(\sum_{j=1}^{i}\bz_j) + \ell^\star_i(\bz_i)\leq \psi^\star_{i+1,V}(\sum_{j=1}^{i-1}\bz_j-\bg_i) + \ell^\star_i(\bg_i)$\\
\hline
\end{tabular}
}
\end{table*}

Our motivation stems from the fact that in practical cases, all the variants that use full losses offer a big advantage in terms of empirical performance at the cost of a higher computational complexity. On the theoretical side, the situation is not so clear given that in the worst case using the full losses can be equivalent to their linearized version, as it should be clear considering linear losses. In particular, the standard theoretical framework for FTRL does not allow a clear analysis of the implicit case.
Moreover, while for implicit OMD it has been proven that one can achieve lower regret if the temporal variation of the losses is small, it is unclear if the same guarantee can be achieved for FTRL without the computational cost of using full losses.

In this paper, we aim at bridging this gap proposing a \emph{generalized} version of implicit FTRL. We go beyond implicit and linearized updates: \emph{we directly construct the update rule in a way that minimizes an upper bound on the regret}.
Our framework effectively expands the scope of the FTRL framework, fully retaining its coupling between design and analysis. Also, our updates come with a worst-case guarantee to never be worse than the standard linearized ones.

We show the flexibility of our framework recovering known update schemes, like the Mirror-Prox update~\cite{Nemirovski04}, or extending updates specifically designed for OMD to the FTRL case, like the aProx one~\cite{AsiD19}. Moreover, for the first time, we show an implicit version of FTRL that recovers the temporal variation bound of implicit OMD~\cite{CampolongoO20}, but with the same computational complexity of implicit OMD.

\paragraph{Related Work}
While there are many works on implicit mirror descent in both the online and offline setting~\citep[see, e.g.,][]{Moreau65,Martinet70,Rockafellar76,KivinenW97,ParikhB14,CampolongoO20,Shtoff22}, the number of works that deal with implicit updates for FTRL is quite limited. We are only aware of \citet{McMahan10}, which quantifies a gain only for specific regularizers. However, the framework in \citet{McMahan10} is non-constructive in the sense that it is difficult to see how to generalize implicit updates. \citet{JoulaniGS17} extends this last result, but it does not provide a link with the maximization of the dual function that governs the regret upper bound.

The closest approach to our framework is the one of \citet{Shalev-ShwartzS07,ShalevS07},  which develop a theory of FTRL updates as maximization of a dual function. However, their framework is limited to a specific shape of regularizers and it does not deal with implicit updates.

For implicit OMD, \citet{CampolongoO20} showed that implicit
updates give rise to regret guarantees that depend on the temporal variability of the losses, so that constant regret is achievable if the variability of the losses is zero. They suggest that FTRL with full losses can achieve the same guarantee, but they also point out that given its computational complexity it would be ``not worth pursuing.''
Here, we show how to achieve the same bound of implicit OMD with our generalized implicit FTRL, while retaining the same computational complexity of implicit OMD.

Proximal updates on truncated linear models were introduced in \citet{AsiD19} for the OMD algorithm.
\citet{ChenLO22} used gradient flow on the same truncated linear models with a coin-betting algorithm~\cite{OrabonaP16}, but their approach does not seem to satisfy a regret guarantee.
\citet{ChenCO22} have used truncated linear models in an FTRL-based parameter-free algorithm~\cite{OrabonaP21} with a novel decomposition of the regret. However, their approach is ad-hoc is it seems difficult to generalize it.

\section{Definitions and Basic Tools}
We define here some basic concepts and tools of convex analysis, we refer the reader to, e.g., \citet{Rockafellar70,BauschkeC11} for a complete introduction to this topic. We will consider extended value function that can assume infinity values too.
A function $f$ is \emph{proper} if it is nowhere $-\infty$ and finite somewhere.
A function $f:V \subseteq \R^d \rightarrow [-\infty, +\infty]$ is \emph{closed} if $\{\bx: f(\bx) \leq \alpha\}$ is closed for every $\alpha \in \R$.
For a proper function $f:\R^d \rightarrow (-\infty, +\infty]$, we define a \emph{subgradient} of $f$ in $\bx \in \R^d$ as a vector $\bg \in \R^d$ that satisfies $f(\by)\geq f(\bx) + \langle \bg, \by-\bx\rangle, \ \forall \by \in \R^d$.
We denote the set of subgradients of $f$ in $\bx$ by $\partial f(\bx)$.
The \emph{indicator function of the set $V$}, $i_V:\R^d\rightarrow (-\infty, +\infty]$, has value $0$ for
$\bx \in V$ and $+\infty$ otherwise.
We denote the \emph{dual norm} of a norm $\|\cdot\|$ by $\|\cdot\|_\star$.
A proper function $f : \R^d \rightarrow (-\infty, +\infty]$ is \emph{$\mu$-strongly convex} over a convex set $V \subseteq \interior \dom f$ w.r.t. $\|\cdot\|$ if $\forall \bx, \by \in V$ and $\forall \bg \in \partial f(\bx)$, we have $f(\by) \geq f(\bx) + \langle \bg , \by - \bx \rangle + \frac{\mu}{2} \| \bx - \by \|^2$.
A function $f:V \rightarrow \R$, differentiable in an open set containing $V$, is \emph{$L$-smooth} w.r.t. $\|\cdot\|$ if  $f(\by) \leq f(\bx) + \langle \nabla f(\bx) , \by - \bx \rangle + \frac{M}{2} \| \bx - \by \|^2$ for all $\bx, \by \in V$.
For a function $f: \R^d\rightarrow [-\infty,\infty]$, we define the \emph{Fenchel conjugate} $f^\star:\R^d \rightarrow [-\infty,\infty]$ as $f^\star(\btheta) = \sup_{\bx \in \R^d} \ \langle \btheta, \bx\rangle - f(\bx)$.
From this definition, we immediately have the Fenchel-Young inequality: $f(\bx) + f^\star(\btheta) \geq \langle \btheta, \bx\rangle, \ \forall \bx, \btheta$.

We will also make use of the following properties of Fenchel conjugates.
\begin{theorem}[{\citep[Theorem~5.7]{Orabona19}}]
\label{thm:props_fenchel}
Let $f:\R^d \rightarrow (-\infty,+\infty]$ be proper. Then, the following conditions are equivalent:
\begin{enumerate}[(a)]
\setlength{\itemsep}{0pt}%
\setlength{\parskip}{0pt}
\vspace{-0.35cm}
\item $\btheta \in \partial f(\bx)$.
\item $\langle \btheta, \by\rangle - f(\by)$ achieves its supremum in $\by$ at $\by=\bx$.
\item $f(\bx)+f^\star(\btheta)=\langle \btheta,\bx\rangle$.
\vspace{-0.35cm}
\end{enumerate}
Moreover, if $f$ is also convex and closed, we have an additional equivalent condition
\begin{enumerate}[(a)]
\setcounter{enumi}{3}
\setlength{\itemsep}{0pt}%
\setlength{\parskip}{0pt}
\vspace{-0.35cm}
\item $\bx \in \partial f^\star(\btheta)$. 
\end{enumerate}
\end{theorem}

\begin{theorem}[{\citep[Theorem~6.11]{Orabona19}}]
\label{thm:duality}
Let $\psi:\R^d \rightarrow (-\infty, +\infty]$ be a proper, closed, convex function, and $\dom \partial \psi$ be non-empty. Then, $\psi$ is $\lambda>0$ strongly convex w.r.t. $\|\cdot\|$ iff $\psi^\star$ is $\frac{1}{\lambda}$-smooth w.r.t. $\|\cdot\|_\star$ on $\R^d$.
\end{theorem}

\section{Generalized Implicit FTRL}
\label{sec:mainres}

In this section, we introduce our novel generalized formulation of the implicit FTRL algorithm.
The main idea is to depart from the implicit or linearized updates,  and directly design updates that improve the upper bound on the regret. 
More in detail, the basic analysis of most of the online learning algorithms is based on the definition of subgradients:
\begin{equation}
\label{eq:subgradient_ineq}
\ell_t(\bx_t) - \ell_t(\bu)
\leq \langle \bg_t, \bx_t-\bu\rangle, \ \forall \bg_t \in \partial \ell_t(\bx_t)~.
\end{equation}
This allows to study the regret on the linearized losses as a proxy for the regret on the losses $\ell_t$.
However, we can do better. We introduce a new fundamental and more general strategy: using the Fenchel-Young inequality, we have
\[
\ell_t(\bx_t) - \ell_t(\bu)  
\leq  \ell_t(\bx_t) - \langle \bz_t,\bu\rangle + \ell_t^\star(\bz_t), \ \forall \bz_t~.
\]
In particular, the algorithm will choose $\bz_t$ to make a certain upper bound involving this quantity to be tighter.
This is a better inequality than \eqref{eq:subgradient_ineq} because when we select $\bz_t=\bg_t \in \partial \ell_t(\bx_t)$, using Theorem~\ref{thm:props_fenchel}, we recover \eqref{eq:subgradient_ineq}. So, this inequality subsumes the standard one for subgradients, but, using $\bz_t \in \ell_t(\bx_{t+1})$, it also subsumes the similar inequality used in the implicit case, as we show in Section~\ref{sec:prox}. Moreover, we will see in Section~\ref{sec:aprox} that it covers cases where $\bz_t$ is \emph{not} a subgradient of $\ell_t$. 

The analysis shows that the optimal setting of $\bz_t$ is the one that minimizes the function 
\begin{equation}
\label{eq:h}
H_t(\bz)
\triangleq\psi^\star_{t+1,V}(\btheta_{t}-\bz) + \ell^\star_t(\bz)
\end{equation}
or
\begin{equation}
\label{eq:hprime}
H'_t(\bz)
\triangleq\psi^\star_{t,V}(\btheta_{t}-\bz_t) + \ell^\star_t(\bz),
\end{equation}
where $\psi_{t,V}$ is the restriction of the regularizer used at time $t$ on the feasible set $V$, i.e., $\psi_{t,V}\triangleq\psi_t+i_V$.
However, we can show that any setting of $\bz_t$ that guarantees $H(\bz_t)< H(\bg_t)$ (or $H'(\bz_t)< H'(\bg_t)$) guarantee a strict improvement in the worst-case regret w.r.t. using the linearized losses.

One might wonder why the need for two different updates using $H_t$ or $H'_t$. The reason is that when using time-varying regularizers that depend on the data, like in the FTRL version of AdaGrad~\citep{McMahanS10,DuchiHS11}, if $\lambda_{t+1}$ depends on $\bz_t$ it might make the calculation of the update particularly difficult. This can be avoided using the update involving $H'_t$.

\begin{algorithm}[t]
\caption{Generalized Implicit FTRL}
\label{alg:giftrl}
\begin{algorithmic}[1]
{
    \REQUIRE{Non-empty closed set $V\subseteq \R^d$, a sequence of regularizers $\psi_1, \dots, \psi_T : \R^d \rightarrow (-\infty, +\infty]$}
    \STATE{$\btheta_1=\boldsymbol{0}$}
    \FOR{$t=1$ {\bfseries to} $T$}
    \STATE{Output $\bx_t \in \argmin_{\bx \in V} \ \psi_t(\bx) - \langle \btheta_t, \bx\rangle$}
    \STATE{Receive $\ell_t:V \rightarrow \R$ and pay $\ell_t(\bx_t)$}
    \STATE{Set $\bg_t \in \partial \ell_t(\bx_t)$}
    \STATE{Set $\bz_t$ such that $H_t(\bz_t)\leq H_t(\bg_t)$ or $H'_t(\bz_t)\leq H'_t(\bg_t)$ where $H_t$ and $H'_t$ are defined in \eqref{eq:h} and \eqref{eq:hprime}}
    \STATE{Set $\btheta_{t+1}=\btheta_t-\bz_t$}
    \ENDFOR
}
\end{algorithmic}
\end{algorithm}

Once we have the $\bz_t$, we treat them as the subgradient of surrogate linear losses.
So, putting it all together, Algorithm~\ref{alg:giftrl} shows the final algorithm.
We now show a regret guarantee for this algorithm. First, we state a general Lemma and then instantiate it in a few interesting cases.

\begin{theorem}
\label{thm:main}
Let $V\subseteq \R^d$ be closed and non-empty and $\psi_t:V \rightarrow \R$.
With the notation in Algorithm~\ref{alg:giftrl}, define by $F_t(\bx) = \psi_{t}(\bx) + \sum_{i=1}^{t-1} \langle \bz_i, \bx\rangle$, so that $\bx_t \in \argmin_{\bx \in V} \ F_{t}(\bx)$. 
Finally, assume that $\argmin_{\bx \in V} \ F_{t}(\bx)$ and $\partial \ell_t(\bx_t)$ are not empty for all $t$.
\begin{itemize}
\item For any $\bz_t \in\R^d$ and any $\bu \in \R^d$, we have
\begin{align*}
&\Regret_T(\bu) \leq \psi_{T+1}(\bu) - \min_{\bx \in V} \ \psi_{1}(\bx)\\
&\quad +\sum_{t=1}^T [\psi^\star_{t+1,V}(\btheta_{t}-\bg_t) - \psi^\star_{t,V}(\btheta_t) + \langle \bx_t, \bg_t\rangle-\delta_t] \\
&\quad + F_{T+1}(\bx_{T+1}) - F_{T+1}(\bu),
\end{align*}
where $\delta_t \triangleq H_t(\bg_t)-H_t(\bz_t)$.
\item If $\psi_{t+1}(\bx) \geq \psi_t(\bx)$ for any $\bx \in V$, then, for any $\bz_t \in \R^d$, we have

\begin{align*}
&\Regret_T(\bu)\leq\psi_{T+1}(\bu) - \min_{\bx \in V} \ \psi_{1}(\bx)\\
&\quad   +\sum_{t=1}^T [\psi^\star_{t,V}(\btheta_{t}-\bg_t) - \psi^\star_{t,V}(\btheta_t) + \langle \bx_t, \bg_t\rangle-\delta_t] \\
&\quad + F_{T+1}(\bx_{T+1}) - F_{T+1}(\bu),
\end{align*}
where $\delta'_t \triangleq H'_t(\bg_t)-H'_t(\bz_t)$.
\end{itemize}
\end{theorem}
\begin{proof}
The proof is composed of simple but not obvious steps.
The first important observation is that the definition of $\bx_t$ in the algorithm corresponds exactly to the one of FTRL on the linear losses $\langle \bz_t, \cdot\rangle$. Hence, we can use the FTRL equality in~\citet[Lemma 7.1]{Orabona19}:
\begin{align*}
&-\sum_{t=1}^T \langle \bz_t,\bu\rangle\\
&\quad =  + \sum_{t=1}^T [F_t(\bx_t) - F_{t+1}(\bx_{t+1})]\\
&\qquad \psi_{T+1}(\bu) - \min_{\bx \in V} \ \psi_{1}(\bx)  + F_{T+1}(\bx_{T+1}) - F_{T+1}(\bu),
\end{align*}
where we have simplified the terms $\langle \bz_t, \bx_t\rangle$ on both sides.

Now, use Fenchel-Young inequality, to have $\langle \bz_t,\bu\rangle \leq \ell_t(\bu) + \ell_t^\star(\bz_t)$.
Hence, we have
\begin{align*}
-\sum_{t=1}^T \ell_t(\bu)
&\leq  \sum_{t=1}^T [F_t(\bx_t) - F_{t+1}(\bx_{t+1}) + \ell_t^\star(\bz_t)] \\
&\quad +\psi_{T+1}(\bu) - \min_{\bx \in V} \ \psi_{1}(\bx) \\
&\quad + F_{T+1}(\bx_{T+1}) - F_{T+1}(\bu)~.
\end{align*}
Observe that 
\begin{align*}
F_t(\bx_t) 
&= \min_{\bx \in V} \ \psi_{t}(\bx) + \sum_{i=1}^{t-1} \langle \bz_i, \bx\rangle\\
&= - \max_{\bx \in V} \ \langle \btheta_t, \bx\rangle - \psi_{t}(\bx) 
= - \psi^\star_{t,V}(\btheta_t)~.
\end{align*}
In the same way, we have $-F_{t+1}(\bx_{t+1}) = \psi^\star_{t+1,V}(\btheta_{t+1})$.
Also, for any $\bg_t \in \partial \ell_t(\bx_t)$, by Theorem~\ref{thm:props_fenchel} we have 
$\ell_t^\star(\bg_t) = \langle \bx_t, \bg_t\rangle - \ell_t(\bx_t)$.
Hence, each term in the sum can be written as
\begin{align*}
&F_t(\bx_t) - F_{t+1}(\bx_{t+1}) + \ell_t^\star(\bz_t) \\
&\quad = \psi^\star_{t+1,V}(\btheta_{t+1}) - \psi^\star_{t,V}(\btheta_t) + \ell_t^\star(\bz_t)\\
&\quad = H_t(\bz_t) - \psi^\star_{t,V}(\btheta_t)~.
\end{align*}
Now, we just add and subtract $H_t(\bg_t) = \psi^\star_{t+1,V}(\btheta_t - \bg_t) +\langle \bg_t,\bx_t\rangle - \ell_t(\bx_t) $ to obtain the stated bound.

The second case is similar. We just have to observe that if $\psi_{t+1,V} \geq \psi_{t,V}$, then $\psi^\star_{t+1,V} \leq \psi^\star_{t,V}$.
Hence, each term in the sum can be upper bounded as
\begin{align*}
&F_t(\bx_t) - F_{t+1}(\bx_{t+1}) + \ell_t^\star(\bz_t) \\
&\quad \leq \psi^\star_{t,V}(\btheta_{t+1}) - \psi^\star_{t,V}(\btheta_t) + \ell_t^\star(\bz_t)\\
&\quad = H'_t(\bz_t) - \psi^\star_{t,V}(\btheta_t)~.
\end{align*}
As before, adding and subtracting $H'_t(\bg_t) = \psi^\star_{t,V}(\btheta_{t}-\bg_t) + \langle \bx_t, \bg_t\rangle - \ell_t(\bx_t)$ gives the stated bound.
\end{proof}

The Theorem is stated with very weak assumption to show its generality, but it is immediate to obtain concrete regret guarantees just assuming, for example, strongly convex regularizers and convex and Lipschitz losses and using well-known methods as \citet[Lemma~7.8]{Orabona19}

However, we can already understand why this is an interesting guarantee. Let's first consider the case that $\bz_t=\bg_t$. In this case, we exactly recover the linearized FTRL algorithm. Even the guarantee in the Theorem exactly recovers the best known one~\citep[Corollary 7.9]{Orabona19}, with $\delta_t=0$ and $\delta'_t=0$. Now, if we set $\bz_t$ such that $H_t(\bz_t)< H_t(\bg_t)$ or $H'_t(\bz_t)< H'_t(\bg_t)$ we will have that $\delta_t>0$ or $\delta'_t>0$. Hence, in each single term of the sum we have a negative factor that makes the regret bound smaller. While it might be difficult to give a lower bound to $\delta_t$ and $\delta'_t$ without additional assumptions, the main value of this analysis is in giving a \emph{unifying way to design generalized implicit updates for FTRL}. In fact, in the next sections we will show a number of possibilities that this framework enables.

Next, we will gain more understanding on the updates in Algorithm~\ref{alg:giftrl}, comparing them to implicit OMD.

\subsection{Comparison with Implicit Online Mirror Descent}
\label{sec:prox}

In this section, we show that when $\bz_t$ is set to minimize $H_t(\bz)$ or $H'_t(\bz)$, we recover different variants of implicit updates.

Assume that the $\ell_t$ are closed and convex. Also, assume that $\psi^\star_{t,V}$ is differentiable, that is true, for example, when $\psi_t$ is strongly convex by Theorem~\ref{thm:duality}. Then, observe that by the first-order optimality condition and Theorem~\ref{thm:props_fenchel}, we have
\begin{align}
\bz_t &= \argmin_{\bz} \ H_t(\bz) \nonumber \\
&\Leftrightarrow \nabla \psi^\star_{t+1,V}(\btheta_{t}-\bz_t) \in \partial \ell^\star_t(\bz_t) \nonumber \\
& \Leftrightarrow \bz_t \in  \partial \ell_t(\nabla \psi^\star_{t+1,V}(\btheta_{t}-\bz_t)) 
= \partial \ell_t(\bx_{t+1})~. \label{eq:cond_update_exact}
\end{align}
Hence, in this case, we have that the optimal $\bz_t$ is the gradient at the \emph{next} point $\bx_{t+1}$. This is exactly what happens in the implicit updates. 

Under the same assumptions, we also have
\begin{align}
\bz_t = \argmin_{\bz} \ H'_t(\bz) 
&\Leftrightarrow \nabla \psi^\star_{t,V}(\btheta_{t}-\bz_t) \in \partial \ell^\star_t(\bz_t) \nonumber \\
& \Leftrightarrow \bz_t \in  \partial \ell_t(\nabla \psi^\star_{t,V}(\btheta_{t+1}))~. \label{eq:cond_update_inexact}
%= \partial \ell_t(\nabla \psi^\star_{t,V}(\btheta_{t+1}))
\end{align}
In this other case, the update also has an implicit flavor but the subgradient is queried on a point different from the next point, where the difference depends on how much $\nabla \psi^\star_{t, V}$ differs from $\nabla \psi^\star_{t+1, V}$.

Let's see this connection even more precisely, considering \emph{proximal updates}.
Hence, for simplicity, let's consider the case that $V=\R^d$, similar considerations hold in the constrained case.
Consider the case that $\psi_t(\bx)=\frac{\lambda_t}{2}\|\bx\|_2^2$. In this case, the update can be written with the \emph{proximal operator} of the loss functions. In particular, the proximal operator of $\eta f$, is defined as
\[
\Prox_{\eta f}(\by)
\triangleq \argmin_{\bx \in \R^d} \ \frac{1}{2}\|\bx-\by\|_2^2 + \eta f(\bx)~.
\]
If the function $f$ is differentiable we have that $\Prox_{\eta f}(\by) = \by - \eta \nabla f(\Prox_{\eta f}(\by))$. In words, the proximal update moves by a quantity that depends on the gradient on the updated point. The implicit nature of these updates justifies the name ``implicit updates'' used in the online learning literature.
More generally, we have that $\Prox_{\eta f}(\by) \in \by - \eta \partial f(\Prox_{\eta f}(\by))$. We list some common proximal operators in Appendix~\ref{sec:updates}.

Assuming $\lambda_{t+1}$ does not depend on $\bz_t$, using the proximal operator we can rewrite the update in \eqref{eq:cond_update_exact} as

\begin{align}
\bx_{t+1}
&=\frac{\btheta_{t+1}}{\lambda_{t+1}}
= \Prox_{\frac{\ell_t}{\lambda_{t+1}}}\left(\frac{\btheta_t}{\lambda_{t+1}}\right) \nonumber \\
&= \Prox_{\frac{\ell_t}{\lambda_{t+1}}}\left(\frac{\lambda_t \bx_t}{\lambda_{t+1}}\right)~. \label{eq:prox_ftrl_2}
\end{align}

Similarly, we can rewrite the update in \eqref{eq:cond_update_inexact} as
\begin{align*}
\frac{\btheta_{t+1}}{\lambda_t}
&= \frac{\btheta_{t}}{\lambda_t} - \frac{\bz_t}{\lambda_t}
= \bx_t - \frac{\bz_t}{\lambda_t}
\in \bx_t - \frac{1}{\lambda_t}\partial \ell_t(\nabla \psi^\star_{t,V}(\btheta_{t+1})) \\
&= \bx_t - \frac{1}{\lambda_t}\partial \ell_t\left(\frac{\btheta_{t+1}}{\lambda_t}\right)~.
\end{align*}
Hence, we have that $\frac{\btheta_{t+1}}{\lambda_t} = \Prox_{\frac{\ell_t}{\lambda_t}} (\bx_t)$ and we get
\begin{equation}
\label{eq:prox_ftrl_1}
\bx_{t+1}
= \frac{\btheta_{t+1}}{\lambda_{t+1}} 
= \frac{\lambda_t}{\lambda_{t+1}}\Prox_{\frac{\ell_t}{\lambda_t}} (\bx_t)~.
\end{equation}

It is instructive to compare both updates with the one of Implicit Online Mirror Descent using $\psi(\bx)=\frac{1}{2}\|\bx\|^2_2$ as distance generating function and stepsizes $\frac{1}{\lambda_t}$. In this case, we would update with 
\begin{align}
\bx_{t+1} 
&= \argmin_{\bx} \frac12 \|\bx_t-\bx\|_2^2 + \frac{1}{\lambda_t}\ell_t(\bx) \nonumber \\
&= \Prox_\frac{\ell_t}{\lambda_t}(\bx_t)~. \label{eq:prox_omd}
\end{align}
Comparing \eqref{eq:cond_update_exact} and \eqref{eq:cond_update_inexact} to \eqref{eq:prox_omd}, we see, when $\lambda_t \leq \lambda_{t+1}$ as it is usual, the two updates above shrink a bit towards the zero vector, that is the initial point $\bx_1$, before or after the proximal operator. This shrinking is given by the FTRL update and it is the key difference with 
Implicit OMD update.
The different update also corresponds to a different guarantee: the regret of the generalized implicit FTRL holds for unbounded domains too, while in Implicit OMD with time-varying stepsizes can have linear regret on unbounded domains~\citep{OrabonaP18}. Interestingly, a similar shrinking has been proposed in \citet{FangHPF20} to fix the unbounded issue in OMD.
Clearly, the updates \eqref{eq:cond_update_exact} and \eqref{eq:cond_update_inexact} become equivalent to  \eqref{eq:prox_omd} for $\lambda_t$ constant in $t$, that is exactly the only case when implicit/proximal online mirror descent works for unbounded domains.

\section{Temporal Variability Bound}
\label{sec:temp}

In this section, we quantify the advantage of the generalized implicit FTRL updates in the case of slow temporal variability of the loss functions.

It was observed in \citet{CampolongoO20} that implicit OMD satisfies regret guarantees that depends on the temporal Variability $V_T$:
\[
V_T 
\triangleq \sum_{t=2}^T \max_{x\in V} \ \ell_t(\bx) - \ell_{t-1}(\bx)~.
\]
In \citet[Appendix E]{CampolongoO20} they also show that FTRL with full losses guarantees a similar guarantee, but at a much higher computational price. Indeed, FTRL with full losses requires solving a finite sum optimization problem at each step, whose size increases with the number of iterations. Such computational burden induced \citet{CampolongoO20} to say that such approach is ``not worth of pursuing.''

Here, we show that the Algorithm~\ref{alg:giftrl} can satisfy the same guarantee of implicit OMD with the same computational complexity too.
First, we show the following Lemma.
\begin{lemma}
\label{lemma:l1}
Under the assumptions of Theorem~\ref{thm:main}, further assume $V$ to be convex, $\psi_t:V\to \R$ closed, $\lambda_t$-strongly convex w.r.t. $\|\cdot\|$, and subdifferentiable in $V$, $\ell_t$ closed, convex, and subdifferentiable in $V$, and $\lambda_{t+1}\geq \lambda_t$.
Set $\bz_t \in \argmin_{\bz} \ H_t(\bz)$. Then, we have
\begin{align*}
&\Regret_T(\bu)
\leq \psi_{T+1}(\bu) - \min_{\bx \in V} \ \psi_{1}(\bx)  \\
& + \sum_{t=1}^T \left(\ell_t(\bx_t) - \ell_t(\bx_{t+1}) - \frac{\lambda_t}{2}\|\bx_{t+1}-\bx_t\|^2\right), \forall \bu \in V.
\end{align*}
\end{lemma}

\begin{proof}
First of all, the existence and unicity of $\bx_t$ is guaranteed by $\psi_t$ being closed and strongly convex~\citep[see, e.g.,][Theorem 6.8]{Orabona19}.

From Theorem~\ref{thm:props_fenchel}, for any $\bg'_t \in \partial \ell_t(\bx_{t+1})$, we have $\ell_t^\star(\bg'_t) = \langle \bx_{t+1}, \bg'_t\rangle - \ell_t(\bx_{t+1})$.
Hence, from \eqref{eq:cond_update_exact}, we have
\begin{align*}
&\psi^\star_{t+1,V}(\btheta_{t+1}) - \psi^\star_{t,V}(\btheta_t) + \ell_t^\star(\bz_t) \\
& = \psi^\star_{t+1,V}(\btheta_t - \bz_t)-\psi^\star_{t,V}(\btheta_t) +\langle \bx_{t+1}, \bz_t\rangle - \ell_t(\bx_{t+1})~.
\end{align*}
Using this identity, we have
\begin{align*}
&\psi_{t+1,V}^{*}(\btheta_t - \bz_t) - \psi_{t,V}^*(\btheta_t) + \langle \bx_{t+1},\bz_t\rangle\\
&\quad = \langle \btheta_t - \bz_t, \bx_{t+1} \rangle - \psi_{t+1}(\bx_{t+1}) -  \langle \btheta_t,\bx_t \rangle + \psi_t(\bx_t)\\
&\qquad + \langle  \bx_{t+1},\bz_t\rangle\\
&\quad \leq \psi_t(\bx_t) - \psi_t(\bx_{t+1}) + \langle \btheta_t, \bx_{t+1} - \bx_t\rangle~.
\end{align*}
From the first-order optimality condition of $\bx_t$, we have that $\btheta_t \in \partial \psi_t(\bx_t) + \partial i_V(\bx_t)$. Moreover, for all $\bg''_t \in  \partial i_V(\bx_t)$, by definition we have $\langle \bg''_t, \by -\bx_t\rangle \leq 0$ for all $\by \in V$. Hence, for $\bg'_t \in \partial \psi_t(\bx_t)$ and $\bg''_t \in \partial i_V(\bx_t)$ such that $\btheta_t=\bg'_t+\bg''_t$, we have
\begin{align*}
&\psi_t(\bx_t) - \psi_t(\bx_{t+1}) + \langle \btheta_t, \bx_{t+1} - \bx_t\rangle\\
&\quad = \psi_t(\bx_t) - \psi_t(\bx_{t+1}) + \langle \bg'_t+\bg''_t, \bx_{t+1} - \bx_t\rangle \\
&\quad \leq -\frac{\lambda_t}{2}\|\bx_{t+1}-\bx_t\|^2,
\end{align*}
where in the inequality we also used the strong convexity of $\psi_t$.
Using this inequality in Theorem~\ref{thm:main} and summing over time, we have
\begin{align*}
&\sum_{t=1}^T \ell_t(\bx_{t+1})-\sum_{t=1}^T \ell_t(\bu) \\
&\quad \leq \psi_{T+1}(\bu) - \min_{\bx \in V} \ \psi_{1}(\bx) - \sum_{t=1}^T \frac{\lambda_t}{2}\|\bx_{t+1}-\bx_t\|^2~.
\end{align*}
By adding and subtracting $\sum_{t=1}^T \ell_t(\bx_t)$ to both sides and reordering the terms, we have the stated bound.
\end{proof}
This Lemma mirrors Theorem~5.2 in \citet{CampolongoO20}, with the important difference that here we do not need the Bregman divergence to be bounded on the feasible set $V$, thanks to the use of FTRL instead of OMD.
We can now state the immediate corollary on a regret bound that depends on the temporal variation.
\begin{cor}
Under the assumptions of Lemma~\ref{lemma:l1}, for any $\bu \in V$, we have
\begin{align*}
\Regret_T(\bu)
&\leq \psi_{T+1}(\bu) - \min_{\bx \in V} \ \psi_{1}(\bx) \\
&\quad + \ell_1(\bx_1) - \ell_T(\bx_{T+1}) + V_T~.
\end{align*}
\end{cor}
From this result, following \cite{CampolongoO20}, it is relatively easy to obtain the following adaptive regret guarantee. 
The only difficulty is the fact that we need $\psi_{t+1}$ to be independent of $\bz_t$ to have a simpler update rule. We solve this problem using an increasing regularizer that is ``behind of two steps''.
In this way, we have that $\lambda_{t+1}$ depends on quantities that are all known at the beginning of round $t$.
The proof is in Appendix~\ref{sec:prof_temporal}.
\begin{cor}
\label{cor:temporal}
Under the assumptions of Lemma~\ref{lemma:l1}, further assume $\|\bg_t\|_\star \leq G$ for all $t$. 
Define $\gamma_t = \ell_t(\bx_t)-\ell_t(\bx_{t+1})-\frac{\lambda_t}{2}\|\bx_{t+1}-\bx_t\|^2$ and $\lambda_t = \frac{1}{\beta^2}\left(G\beta+\sum_{i=1}^{t-2} \gamma_i\right)$. Assume that $\psi$ is closed and $1$-strongly convex w.r.t. $\|\cdot\|$ and set $\psi_t=\lambda_t \psi$.
Then, for any $\bu \in V$, we have
\begin{align*}
\Regret_T(\bu)
&\leq \min\left(\frac{1}{\beta}(\ell_1(\bx_1) - \ell_T(\bx_{T+1}) + V_T),\right.\\
&\quad \left.G + \sqrt{\frac{5}{4}\sum_{t=1}^T \|\bg_t\|^2_\star}\right)\left(\frac{\psi(\bu)}{\beta}+\beta\right)~.
\end{align*}
\end{cor}

\section{Two-step Updates}
\label{sec:twostep}

The choice of $\bz_t$ that minimizes the regret upper bound requires solving the optimization problem $\min_{\bz} \ H(\bz)$ or $\min_{\bz} \ H'(\bz)$. We have seen in Section~\ref{sec:prox} that this corresponds to (some variant) of a implicit/proximal update and, depending on $\ell_t$, it can be of difficult calculation. However, as we said, any choice better than $\bg_t$ will cause a provable gain. Hence, a viable solution is to \emph{approximately} solve for the optimal $\bz_t$. 

Here, we propose a simple approximation: set $\bz_t$ as
\begin{equation}
\label{eq:two_step_h}
\bz_t \in \partial \ell_t(\nabla \psi^\star_{t+1,V}(\btheta_t-\bg_t))
\end{equation}
or as
\begin{equation}
\label{eq:two_step_hprime}
\bz_t \in \partial \ell_t(\nabla \psi^\star_{t,V}(\btheta_t-\bg_t))~.
\end{equation}
In words, we set $\bz_t$ to be a subgradient after one fake update. 
This is exactly the approach used in the Mirror-Prox algorithm~\citep{Nemirovski04}, an offline optimization algorithm.
In the next theorem, when the loss functions $\ell_t$ are smooth and the regularizer is chosen appropriately, we show that this choice can be used in the generalized implicit FTRL too and it cannot be worse than using $\bg_t$.

\begin{theorem}
Assume $\psi_t(\bx)$ proper, closed, and $\lambda_t$-strongly convex with respect to $\|\cdot\|$. Assume $\ell_t(\bx)$ closed and $\lambda_t$-smooth w.r.t. $\|\cdot\|_\star$ for all $t$. Then, using \eqref{eq:two_step_h} and assuming $\lambda_{t+1}\geq L_t$, we have $H_t(\bz_t)\leq H_t(\bg_t)$. On the other hand, when using \eqref{eq:two_step_hprime} and assuming $\lambda_t \geq L_t$ we have $H'_t(\bz_t)\leq H'_t(\bg_t)$.
\end{theorem}
\begin{proof}
We only prove that statement for \eqref{eq:two_step_h}, the other one is similar.
We would like to prove that
\begin{align*}
H_t(\bz_t)
&=\psi^\star_{t+1,V}(\btheta_{t}-\bz_t) + \ell^\star_t(\bz_t)\\
&\leq \psi^\star_{t+1,V}(\btheta_{t}-\bg_t) + \ell^\star_t(\bg_t) 
=H_t(\bg_t)~.
\end{align*}  
This is equivalent to prove
\[
\psi^\star_{t+1,V}(\btheta_{t}-\bz_t) -\psi^\star_{t+1,V}(\btheta_{t}-\bg_t) 
\leq  \ell^\star_t(\bg_t) - \ell^\star_t(\bz_t)~.
\]
Given that $\psi_{t+1}(\bx_t)$ is $\lambda_{t+1}$-strongly convex, by Theorem~\ref{thm:duality}, we have $\psi_t^\star(\btheta)$ is $1/\lambda_{t+1}$-smooth with respect to $\|\cdot\|_\star$. By the definition of smoothness, we have
\begin{align*}
&\psi^\star_{t+1,V}(\btheta_{t}-\bz_t) -\psi^\star_{t+1,V}(\btheta_{t}-\bg_t) \\
& \quad \leq \langle \nabla \psi_{t+1}^\star(\btheta_t - \bg_t), \bg_t - \bz_t \rangle 
+ \frac{1}{2\lambda_{t+1}}\| \bg_t - \bz_t \|_\star^2~.
\end{align*}
Given that $\ell_t(\bx_t)$ is $L_t$-smooth w.r.t $\|\cdot\|_\star$, by Theorem~\ref{thm:duality} $\ell_t^\star(\bg)$ is $1/L_t$ strongly convex w.r.t. $\|\cdot\|$. So, by the definition of the strong convexity, we have
\[
\ell^\star_t(\bg_t) - \ell^\star_t(\bz_t)
\geq \langle \bq_t, \bg_t - \bz_t\rangle + \frac{1}{2L_t} \| \bg_t - \bz_t\|_\star^2,
\]
for all $\bq_t \in \partial \ell_t^\star(\bz_t)$.
Defining $\bx'_{t+1} \triangleq \nabla \psi_{t+1}^\star(\btheta_t - \bg_t)$, by Theorem~\ref{thm:props_fenchel}, we have $\bx_{t+1}' \in \partial \ell_t^\star(\bz_t)$. Hence, we can select $\bq_t$ such that $\bx_{t+1}' = \bq_t$.
Finally, using the assumption on $\lambda_{t+1} \geq L_t $, we have the stated bound.
\end{proof}

\section{Going Beyond Subgradients with aProx}
\label{sec:aprox}

Till now, in all the updates we have considered $\bz_t$ was set to be a subgradient of $\ell_t$ in a specific point. In this section, we show that we can go beyond this idea.

\citet{AsiD19} introduced aProx updates, that is proximal updates on surrogate loss functions. In particular, they used truncated linear lower bounds to the loss functions as surrogate functions. These simple surrogates are motivated by the fact that they are strictly better than linear approximation and at the same time they allow writing the proximal update in a closed form. Moreover, they showed empirically that in certain situations the performance of the algorithms becomes much more resistant to the tuning of the stepsizes. 

One might just use the same truncated lower bounds in implicit FTRL, but it would not be clear why this should give any advantage in the theoretical bound. Indeed, even in \citet{AsiD19} it is not completely clear what part of the theory tells us that we should expect a better performance from these updates.

Here, we show how the \emph{updates in the generalized implicit FTRL are actually a generalization of the aProx ones}. In particular, we generalize the aProx updates to arbitrary regularizers and show that all of them satisfy $H_t(\bz_t)\leq H_t(\bg_t)$ and $H'_t(\bz_t)\leq H'_t(\bg_t)$. In words, the aProx updates are guaranteed to be at least as good as the subgradient $\bg_t$ in minimizing the worst-case regret.

In order to consider truncated linear lower bounds to the functions $\ell_t$, in this section we will assume that the loss functions $\ell_t$ are lower bounded. Given that the regret is invariant to additive constants in the losses, without loss of generality we can assume the lower bound to be 0 for all the loss functions. Hence, define the truncated linear model $\hat{\ell}_t:V\to \R$ around $\bx_t$ to be
\[
\hat{\ell}_t(\bx) 
\triangleq \max(\ell_t(\bx_t)+ \langle \bg_t, \bx - \bx_t\rangle,0),
\]
where $\bg_t \in \partial \ell_t(\bx_t)$. For brevity of notation, our notation does not stress the fact that the truncated linear model depends on $\bx_t$ and the specific subgradient $\bg_t$.

To idea to extend aProx to the case of generalized implicit FTRL, we use the truncated linear lower bound in the update of $\bz_t$. So, we define 
\begin{equation}
\label{eq:aprox_h}
\bz_t = \argmin_{\bz} \ \psi^\star_{t+1,V}(\btheta_{t}-\bz_t) + \hat{\ell}^\star_t(\bz_t)
\end{equation}
or
\begin{equation}
\label{eq:aprox_hprime}
\bz_t=\argmin_{\bz} \ \psi^\star_{t,V}(\btheta_{t}-\bz_t) + \hat{\ell}^\star_t(\bz_t)~.
\end{equation}

\begin{theorem}
Assume the loss functions $\ell_t:V\to\R$ to be convex, closed, and subdifferentiable in $V$ for all $t$.
Set $\bz_t$ using \eqref{eq:aprox_h} or \eqref{eq:aprox_hprime}. Then, we have that $H_t(\bz_t)\leq H_t(\bg_t)$ or $H'_t(\bz_t)\leq H'_t(\bg_t)$ respectively.
\end{theorem}
\begin{proof}
We consider the update \eqref{eq:aprox_h}, the other case is very similar and we omit it.

First, we derive some inequalities on the quantities of interest.
From Theorem~\ref{thm:props_fenchel}, given that $\bg_t \in \partial \hat{\ell}_t(\bx_t)$ and $\bg_t \in \partial \ell_t(\bx_t)$ we have both $\ell_t(\bx_t)+\ell^\star(\bg_t)= \langle \bg_t, \bx_t\rangle$ and $\hat{\ell}_t(\bx_t)+\hat{\ell}^\star(\bg_t)= \langle \bg_t, \bx_t\rangle$.
Moreover, given that $\hat \ell_t(\bx)\leq \ell_t(\bx)$ for any $\bx$, we have $\hat{\ell}^\star_t(\bz) \geq \ell^\star_t(\bz)$ for any $\bz$. Finally, by the definition of truncated linear lower bound, we have $\ell_t(\bx_t) = \hat{\ell}_t(\bx_t)$.

Hence, we have
\begin{align*}
&\psi^\star_{t+1,V}(\btheta_{t}-\bz_t) + \ell^\star_t(\bz_t)\\
&\quad \leq \psi^\star_{t+1,V}(\btheta_{t}-\bz_t) + \hat{\ell}^\star_t(\bz_t)\\
&\quad= \min_{\bz} \ \psi^\star_{t+1,V}(\btheta_{t}-\bz) + \hat{\ell}^\star_t(\bz)\\
&\quad\leq \psi^\star_{t+1,V}(\btheta_{t}-\bg_t) + \hat{\ell}^\star_t(\bg_t) \\
&\quad=\psi^\star_{t+1,V}(\btheta_{t}-\bg_t) + \langle \bg_t, \bx_t\rangle- \hat{\ell}_t(\bx_t) \\
&\quad=\psi^\star_{t+1,V}(\btheta_{t}-\bg_t) + \langle \bg_t, \bx_t\rangle- \ell_t(\bx_t)\\
&\quad=\psi^\star_{t+1,V}(\btheta_{t}-\bg_t) + \ell^\star_t(\bg_t) = H_t(\bg_t)~. \qedhere
\end{align*}
\end{proof}

We can also immediately write closed form updates for generalized implicit FTRL with  regularizer $\psi_t(\bx)=\frac{\lambda_t}{2}\|\bx\|^2$, that mirror the ones of aProx. The proof is in Section~\ref{sec:proof_aprox}.
\begin{cor}
\label{cor:aprox}
Set $\psi_{t}=\frac{\lambda_t}{2}\|\bx\|^2_2$ and $\bg_t \in \partial \ell_t(\bx_t)$.
Setting $\bz_t$ as in \eqref{eq:aprox_h}, we have that the update of generalized implicit FTRL is
\[
\bx_{t+1} 
= \frac{\lambda_t}{\lambda_{t+1}}\bx_t - \min\left(\frac{1}{\lambda_{t+1}}, \frac{\ell_t(\bx_t)}{\|\bg_t\|^2}\right) \bg_t~.
\]
On the other hand, setting $\bz_t$ as in \eqref{eq:aprox_hprime}, the update is
\[
\bx_{t+1} 
= \frac{\lambda_t}{\lambda_{t+1}} \bx_t - \min\left(\frac{1}{\lambda_{t+1}}, \frac{\lambda_t}{\lambda_{t+1}} \frac{\ell_t(\bx_t)}{\|\bg_t\|^2}\right) \bg_t~.
\]
\end{cor}

\section{Empirical Evaluation}
\label{sec:exp}

\begin{figure}[t]
\centering
\includegraphics[width=0.4\textwidth]{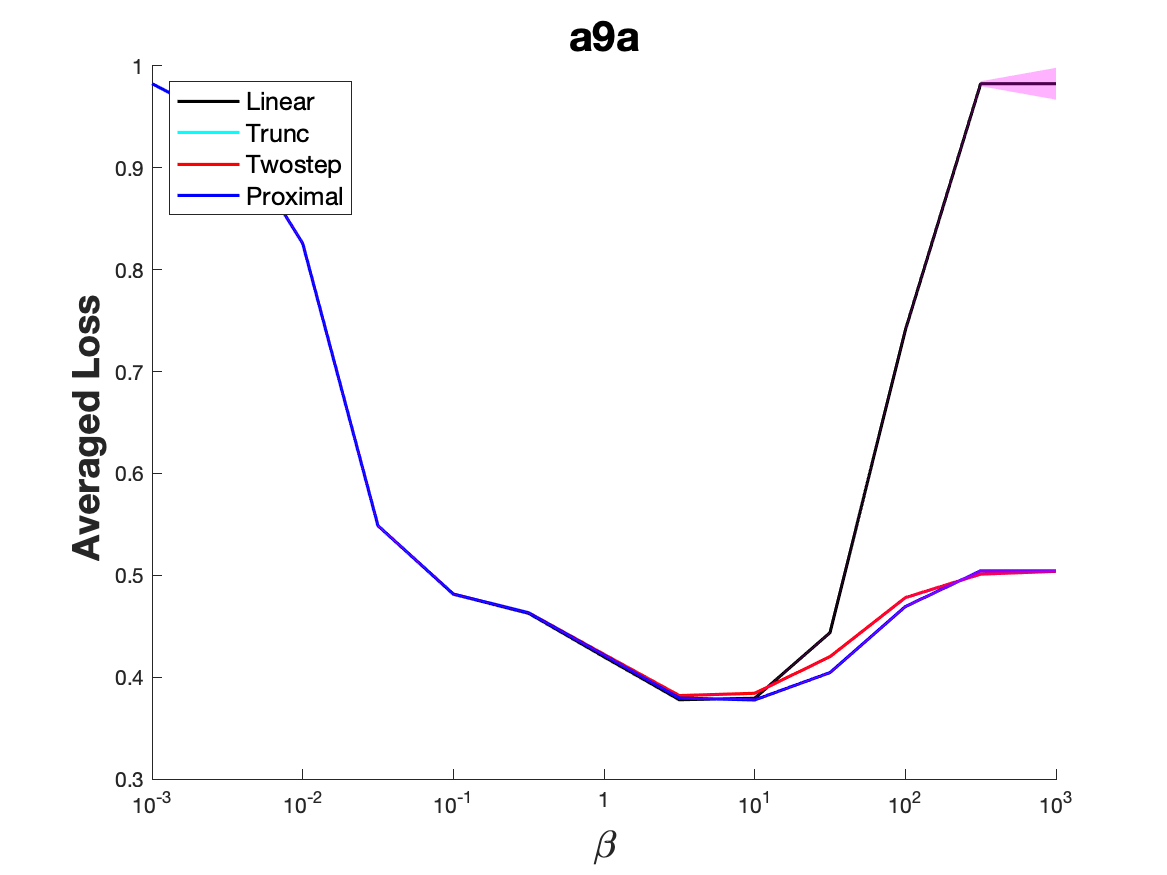} 
\includegraphics[width=0.4\textwidth]{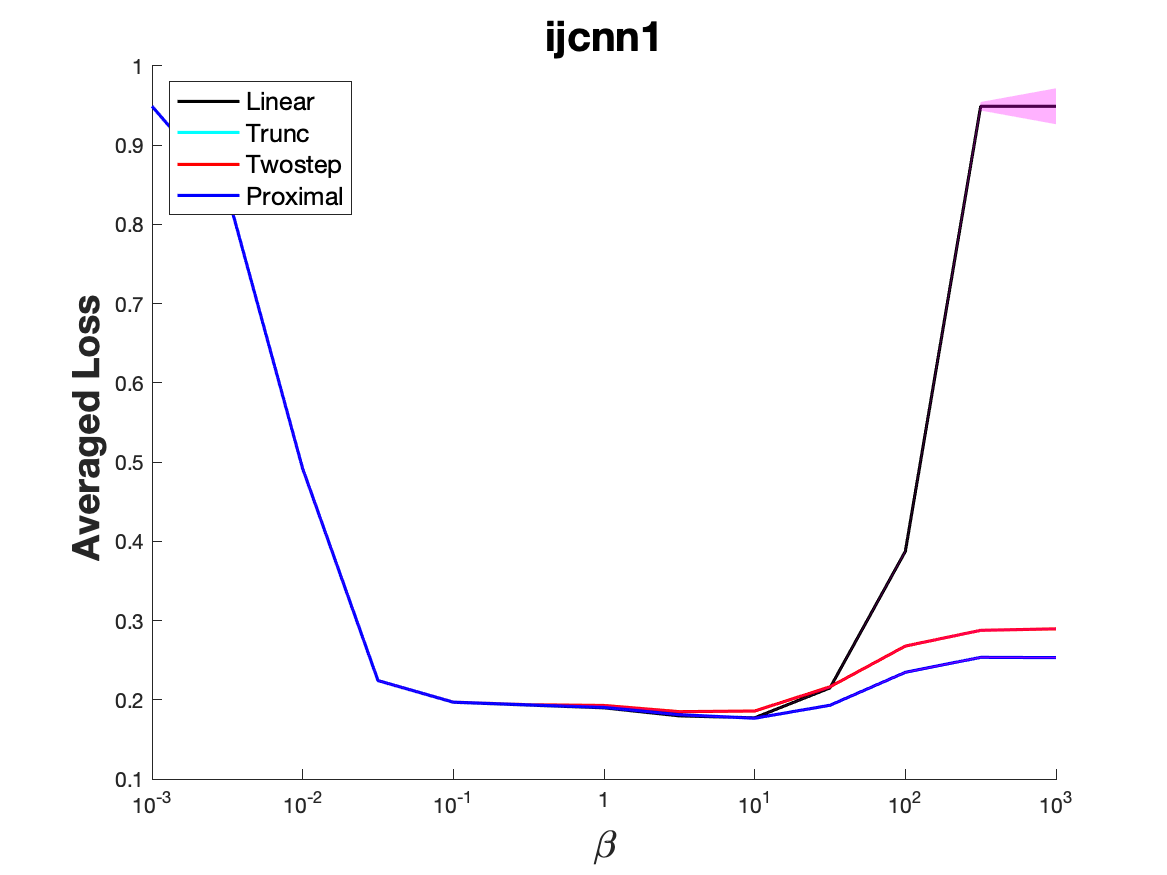}
\includegraphics[width=0.4\textwidth]{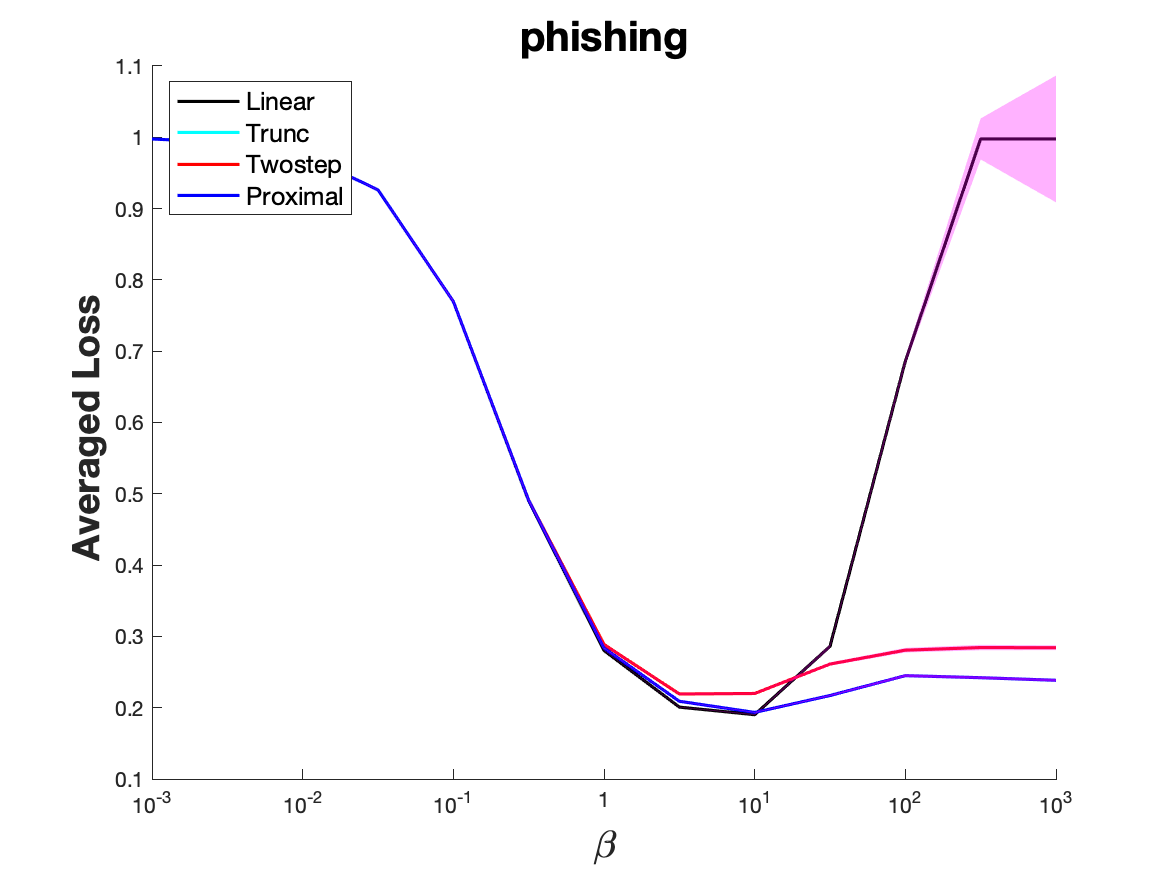}
\vspace{-0.5cm}
\caption{Hinge loss, averaged loss vs. hyperparameter $\beta$.}
\label{fig:hinge_plot}
\end{figure}

\begin{figure}[t]
\centering
\includegraphics[width=0.4\textwidth]{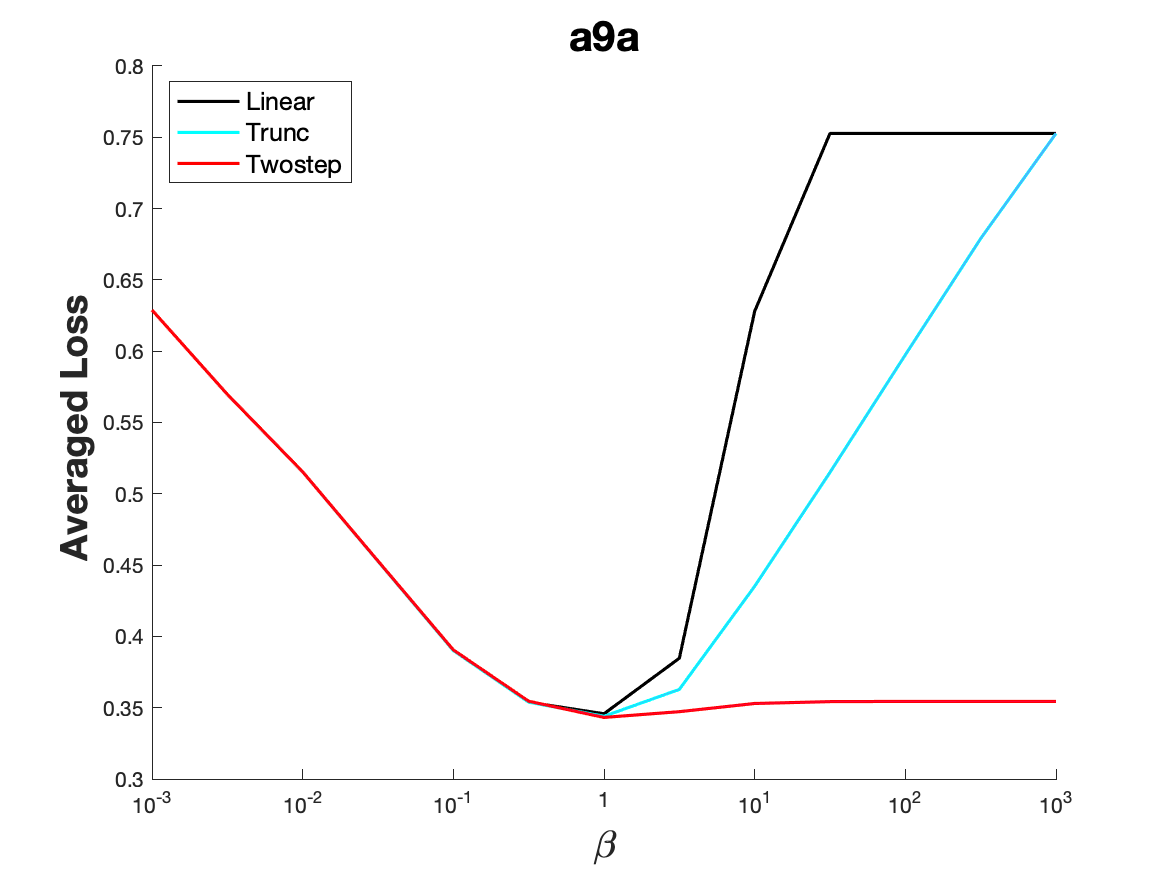}  
\includegraphics[width=0.4\textwidth]{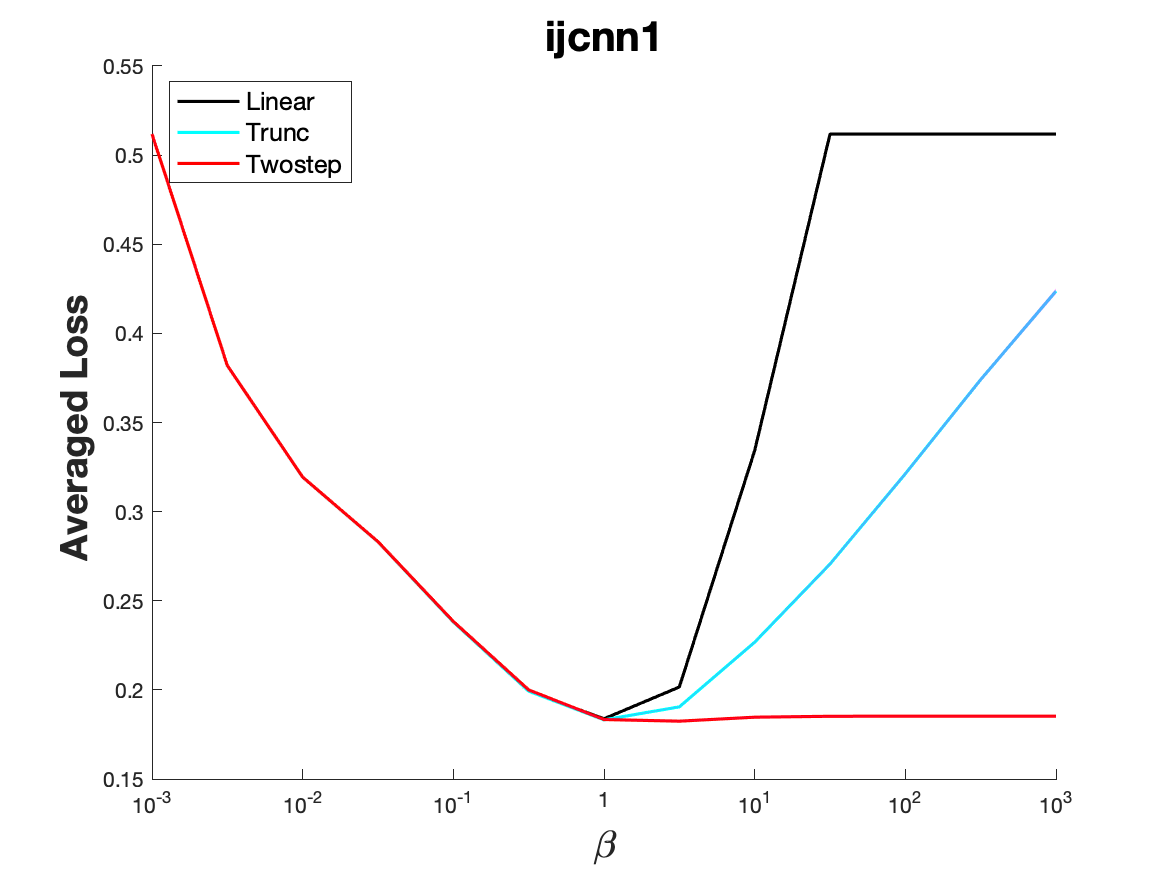} 
\includegraphics[width=0.4\textwidth]{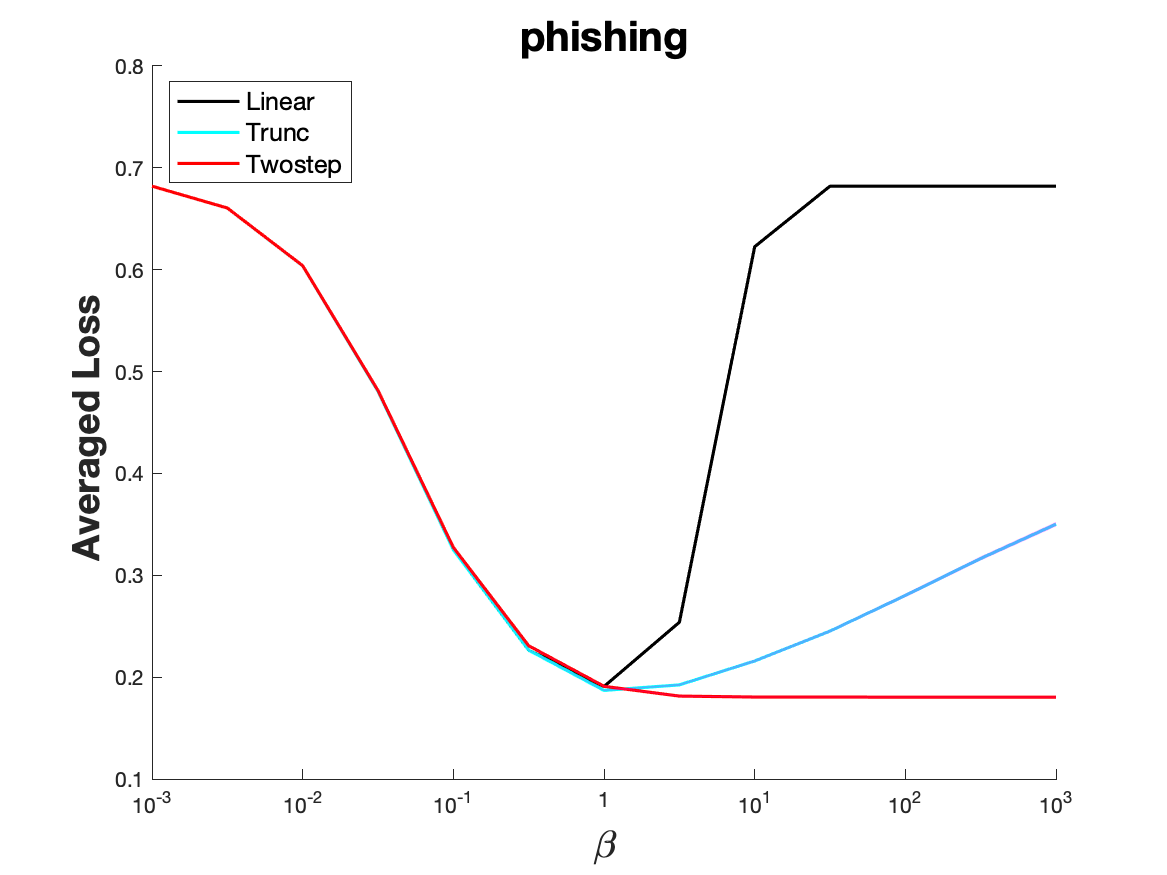} 
\vspace{-0.5cm}
\caption{Logistic loss, averaged loss vs. hyperparameter $\beta$.}
\label{fig:logistic_plot}
\end{figure}

\begin{figure}[ht]
\centering
\includegraphics[width=0.4\textwidth]{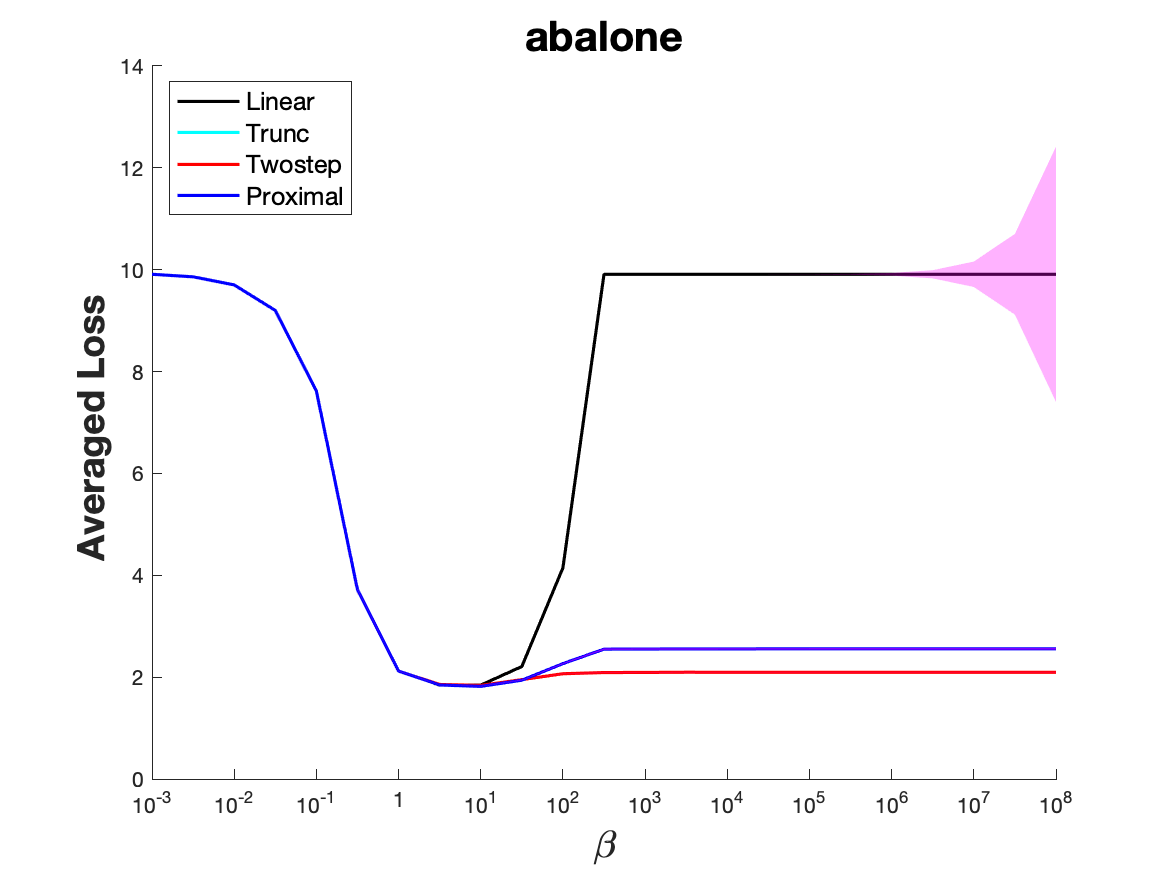}  
\includegraphics[width=0.4\textwidth]{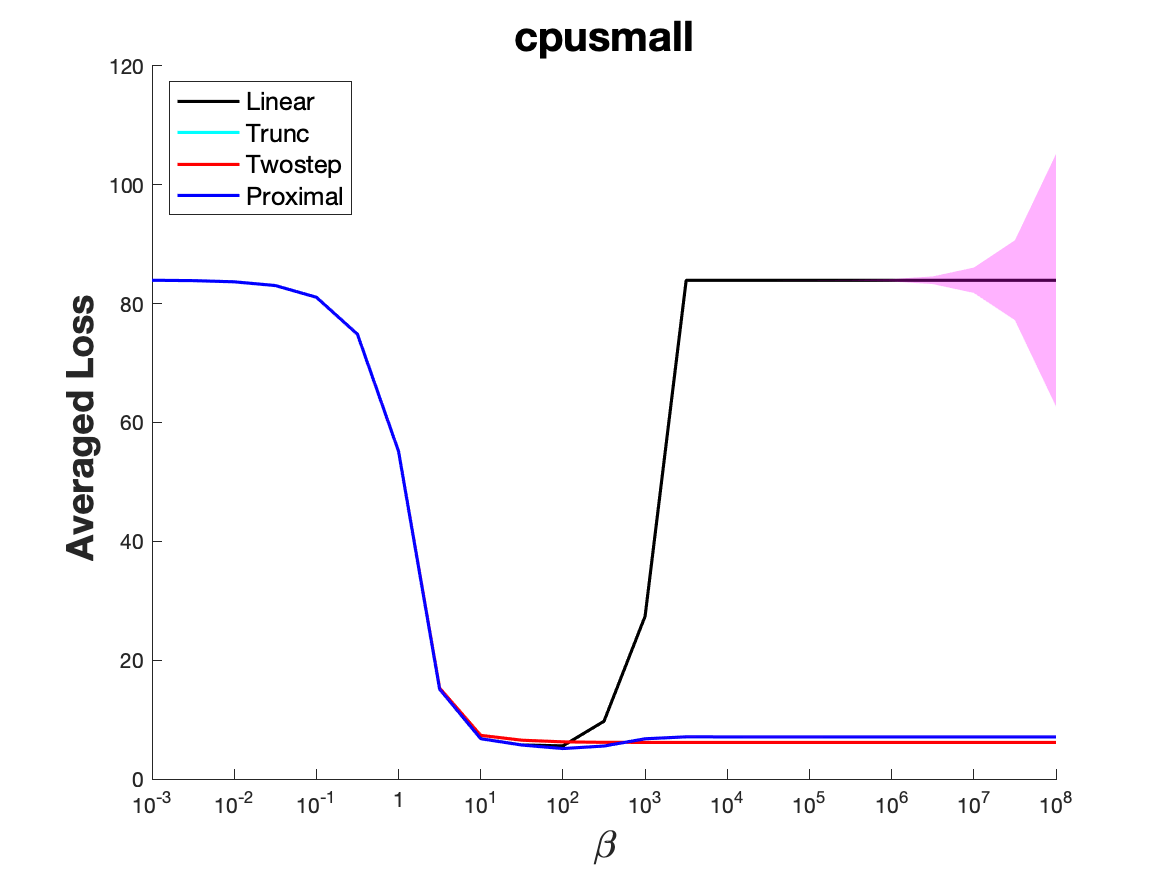} 
\includegraphics[width=0.4\textwidth]{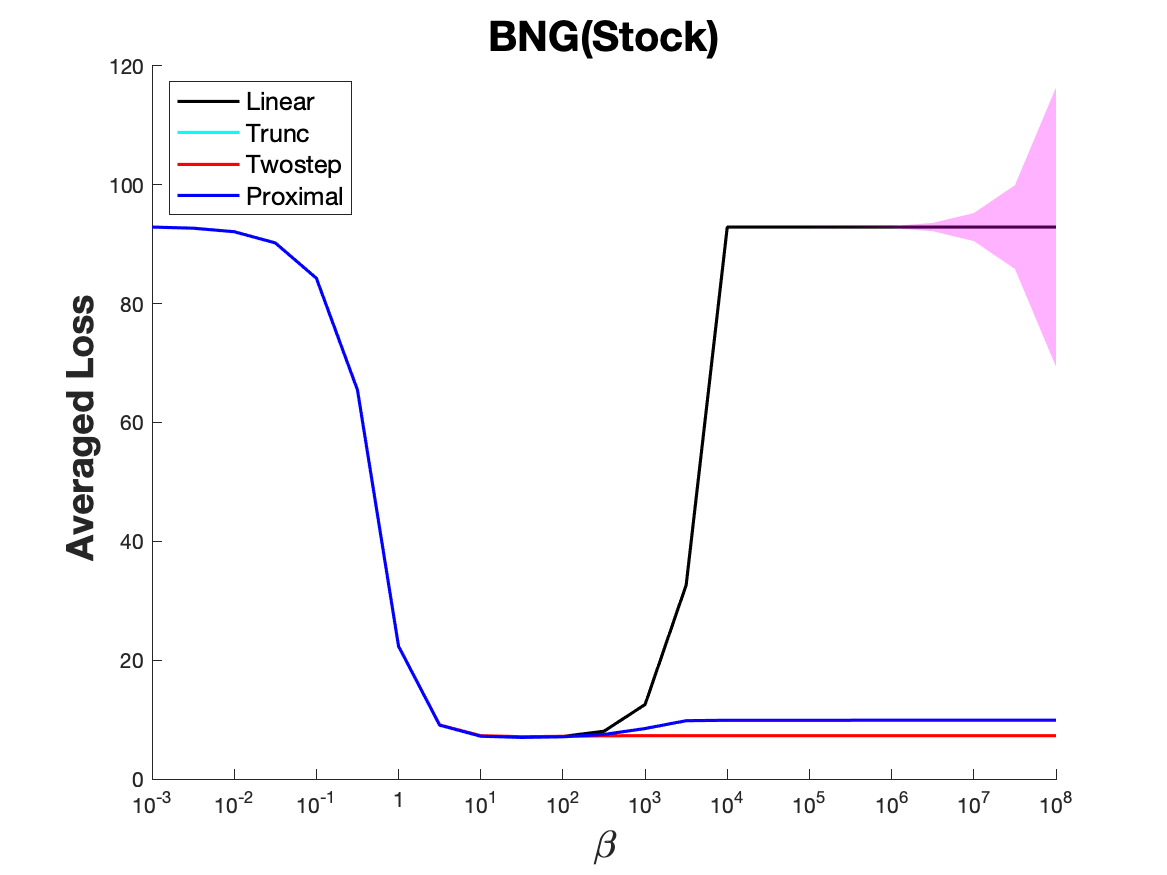} 
\vspace{-0.5cm}
\caption{Absolute loss, averaged loss vs. hyperparameter $\beta$.}
\label{fig:reg_plot}
\end{figure}

As we said, in the worst case scenario any kind of implicit update cannot give any advantage over the usual updates. However, in practice it is well-known that things are vastly different. Hence, in this section, we compare the performance of different choices of $\bz_t$ in Algorithm~\ref{alg:giftrl} when $\psi_t(\bx)=\frac{\lambda_t}{2}\|\bx\|_2^2$.
In particular, we consider:
\begin{itemize}
\setlength{\itemsep}{0pt}%
\setlength{\parskip}{0pt}
\vspace{-0.3cm}
\item FTRL with linearized losses (Linear): $\bz_t=\bg_t$;
\item Implicit FTRL with aProx updates (Trunc): $\bz_t = \min\left\{ 1,\frac{\lambda_t \ell_t(\bx_t)}{\| \bg_t\|^2}\right\}\bg_t$;
\item Implicit FTRL with two-step updates (Twostep): $\bz_t = \partial \ell_t(\bx_t - \bg_t/\lambda_t)$;
\item Implicit FTRL with \eqref{eq:prox_ftrl_2} when the proximal operator has a closed form (Proximal).
\vspace{-0.3cm}
\end{itemize}
We adopt the choice of $\lambda_t$ from Corollary~\ref{cor:temporal}.

We conduct linear prediction experiments on datasets from LibSVM~\citep{ChangL11}. We show here experiments on classification tasks using the hinge loss and the logistic loss, and regression tasks with absolute loss. We normalize the datasets and added a constant bias term to the features. Given that in the online learning setting, we do not have the training data and validation data to tune the $\beta$, we will plot the averaged loss, $\frac{1}{t}\sum_{i=1}^t \ell_i(\bx_i)$, versus different choice of $\beta$, that at the same time show the algorithms' sensitivity to the hyperparameter $\beta$ and their best achievable performance. We consider $\beta \in [10^{-3},10^3]$ for hinge loss and logistic loss, and $\beta \in [10^{-3},10^8]$ for the absolute loss. Each algorithm is run 15 times, we plot the average of the averaged losses and the $95\%$ confidence interval. Note that the confidence intervals so small to be invisible, but for the larger values of the $\beta$ for the Linear updates.

Figure~\ref{fig:hinge_plot} and Figure~\ref{fig:logistic_plot} show the averaged loss versus different selections of hyperparameter $\beta$ for classification tasks with hinge loss and logistic loss respectively. Note that with the hinge loss 
aProx updates and proximal updates are completely equivalent.
In all experiments, FTRL with linearized updates is more sensitive to the setting of $\beta$,  and its performance is almost uniformly worse than all the other generalized implicit updates. This is in line with previous results in \citet{AsiD19} in the offline setting. With the logistic loss, the proximal operator does not have a closed-form solution. In all the classification experiments, the performance of generalized implicit FTRL with two-step updates seems remarkable and a possible viable alternative to aProx.  The confidence intervals for all implicit updates have a width smaller than 0.01, making them too narrow to be visible in the figures. In contrast, when using hinge loss, the performance of FTRL with linear models exhibits significant fluctuations across different repetitions when a large learning rate is used. This observation provides evidence supporting our assertion that the selection of hyperparameter $\beta$ greatly affects the performance of FTRL with linear models, while implicit updates demonstrate robustness.

Figure~\ref{fig:reg_plot} shows that FTRL with linearized updates is very sensitive to the choice of the hyperparameter $\beta$, while the implicit FTRL updates are robust. Again, Implicit FTRL with two-step updates achieves essentially the best performance. The confidence intervals in the regression tasks lead to a similar conclusion as in the classification tasks.

\section{Conclusion and Future Work}
In this work, we propose a new framework: generalized implicit Follow-the-Regularized-Leader. We show that generalized implicit FTRL can not only recover known algorithms, e.g., implicit FTRL and FTRL with linearized losses, but it also provides a theoretical guideline to design new algorithms, such as the extensions of aProx and Mirror-Prox. Indeed, we believe that the main contribution of our work lies precisely in the fact that it provides a unifying framework that is general, flexible, and theoretically grounded.

In the future, we plan to explore further this framework designing new $\bz_t$ with low computational complexity. This is a promising direction because the two-steps update seems to be already a valid alternative to the aProx updates, even if it comes at the computational expense of querying an additional gradient in each round.

\section*{Acknowledgements}
We thank Alex Shtoff for discussion and feedback on a preliminary version of this paper.
Francesco Orabona is supported by the National Science Foundation under the grants no. 2022446 ``Foundations of Data Science Institute'' and no. 2046096 ``CAREER: Parameter-free Optimization Algorithms for Machine Learning''.

\bibliographystyle{icml2023}
\bibliography{../../../learning}

\onecolumn
\appendix
\section{Update Rule for Common Losses}
\label{sec:updates}

In this section, we report the proximal operator of common losses for easy referencing. These formulas are well-known and they can be found, for example, in \citet{CrammerDKSSS06,KulisB10}.

\begin{align*}
\ell_t(\bx)=\max(1-y_t \langle \bs_t, \bx\rangle,0) &\Rightarrow \Prox_{\frac{\ell_t}{\lambda}}(\bx)=\bx+\min\left(\frac{1}{\lambda}, \frac{\max(1-y_t \langle \bs_t, \bx\rangle,0)}{\|\bs_t\|^2}\right) y_t \bs_t\\
\ell_t(\bx)=|\langle \bs_t, \bx\rangle-y_t| &\Rightarrow \Prox_{\frac{\ell_t}{\lambda}}(\bx)=\bx-\min\left(\frac{1}{\lambda}, \frac{|\langle \bs_t, \bx\rangle-y_t|}{\|\bs_t\|^2}\right) \bs_t\\
\ell_t(\bx)=\frac{1}{2}(\langle \bs_t,\bx\rangle - y_t)^2 &\Rightarrow \Prox_{\frac{\ell_t}{\lambda}}(\bx)=\bx-\frac{(\langle \bs_t, \bx\rangle-y_t) \bs_t}{\lambda+  \|\bs_t\|_2^2}~.
\end{align*}

\section{Proof of Corollary~\ref{cor:temporal}}
\label{sec:prof_temporal}

\begin{proof}
From the regret guarantee in Lemma~\ref{lemma:l1}, we have that 
\[
\Regret_T(\bu)
\leq \psi_{T+1}(\bu) + \beta^2 \sum_{t=1}^T \gamma_t
\leq \lambda_{T+1} (\psi(\bu)+\beta^2), \ \forall \bu \in V~.
\]

Now, we upper bound $\sum_{t=1}^T \gamma_t$ in two different ways.
In the first upper bound, we have 
\begin{align*}
\sum_{t=1}^T \gamma_t
&= \sum_{t=1}^T \left(\ell_t(\bx_t)-\ell(\bx_{t+1})-\frac{\lambda_t}{2}\|\bx_{t+1}-\bx_t\|^2\right)
= \ell_1(\bx_1) - \ell_T(\bx_{T+1}) + \sum_{t=2}^T (\ell_t(\bx_t) - \ell_{t-1}(\bx_t))\\
&\leq \ell_1(\bx_1) - \ell_T(\bx_{T+1}) + V_T~.
\end{align*}

For the second upper bound, we have
\begin{align*}
\gamma_t
&= \ell_t(\bx_t)-\ell(\bx_{t+1})-\frac{\lambda_t}{2}\|\bx_{t+1}-\bx_t\|^2
\leq \langle \bg_t, \bx_t-\bx_{t+1}\rangle -\frac{\lambda_t}{2}\|\bx_{t+1}-\bx_t\|^2\\
&\leq \frac{\|\bg_t\|^2_\star}{2\lambda_t} + \frac{\lambda_t}{2}\|\bx_{t+1}-\bx_t\|^2- \frac{\lambda_t}{2}\|\bx_{t+1}-\bx_t\|^2
=\frac{\|\bg_t\|^2_\star}{2\lambda_t}
\leq \beta\frac{\|\bg_t\|_\star}{2},
\end{align*}
where we used Fenchel-Young inequality and the second lower bound is obtained by using the fact that $\lambda_t\geq \lambda_1=\frac{G}{\beta}$.
Hence, we have
\[
\lambda_{t+1}
= \lambda_t +\frac{\gamma_t}{\beta^2}
\leq \lambda_t + \min\left(\frac{\|\bg_t\|_\star}{2\beta},\frac{\|\bg_t\|^2_\star}{2\beta^2\lambda_t}\right)
\]
Using Lemma 6.1 in \citet{CampolongoO20} and taking into account the fact that $\lambda_1=\frac{G}{\beta}$, we have
\[
\lambda_{T+1} \leq \frac{G}{\beta} + \sqrt{\frac{5}{4\beta^2}\sum_{t=1}^T \|\bg_t\|^2_\star}~.
\]
Putting all together, we have the stated bound.
%\[
%\Regret_T(\bu)
%\leq \min\left(\frac{1}{\beta^2}(\ell_1(\bx_1) - \ell_T(\bx_{T+1}) + V_T),\frac{1}{\beta}\left(G + \sqrt{\frac{5}{4}\sum_{t=1}^T \|\bg_t\|^2_\star}\right)\right)\left(\psi(\bu)+\beta^2\right)~.
%\]
\end{proof}

\section{Proof of Corollary~\ref{cor:aprox}}
\label{sec:proof_aprox}

\begin{proof}
The proximal operator of $\frac{\hat{\ell}_t}{\lambda}$ is
\[
\Prox_{\frac{\hat{\ell}_t}{\lambda}} (\bx) 
= \bx - \min\left(\frac{1}{\lambda}, \frac{\ell_t(\bx_t)}{\|\bg_t\|^2}\right) \bg_t~.
\]
Hence, from \eqref{eq:prox_ftrl_1}, we have
\begin{align*}
\bx_{t+1} 
&= \frac{\lambda_t}{\lambda_{t+1}} \left(\bx_t - \min\left(\frac{1}{\lambda_t}, \frac{\ell_t(\bx_t)}{\|\bg_t\|^2}\right) \bg_t\right) 
= \frac{\lambda_t}{\lambda_{t+1}} \bx_t - \min\left(\frac{1}{\lambda_{t+1}}, \frac{\lambda_t}{\lambda_{t+1}} \frac{\ell_t(\bx_t)}{\|\bg_t\|^2}\right) \bg_t~.
\end{align*}
Instead, from \eqref{eq:prox_ftrl_2}, we have
\[
\bx_{t+1} 
= \frac{\lambda_t}{\lambda_{t+1}}\bx_t - \min\left(\frac{1}{\lambda_{t+1}}, \frac{\ell_t(\bx_t)}{\|\bg_t\|^2}\right) \bg_t~. \qedhere
\]
\end{proof}

\end{document}